\def\isarxiv{1}

\ifdefined\isarxiv
\documentclass[11pt]{article}
\else
\documentclass{article} 
\usepackage{microtype} 
\usepackage{graphicx}
\usepackage{subfig}
\usepackage{hyperref}
\usepackage{icml2024}
\fi

\usepackage[utf8]{inputenc} %
\usepackage[T1]{fontenc}    %
\usepackage{url}            %
\usepackage{booktabs}       %
\usepackage{amsfonts}       %
\usepackage{nicefrac}       %
\usepackage{microtype}      %

\usepackage{amsmath}
\usepackage{amsthm}
\usepackage{amssymb}
\usepackage{algorithm}
\usepackage{color}
\usepackage[english]{babel}
\usepackage{grffile}
\usepackage{wrapfig,epsfig}
\usepackage{epstopdf}
\usepackage{url}
\usepackage{color}
\usepackage{epstopdf}
\usepackage{algpseudocode}
\usepackage[T1]{fontenc}
\usepackage{bbm}
\usepackage{comment}
\usepackage{caption}
\usepackage{dsfont}
\usepackage{tikz}

\usepackage{tikz}
\usetikzlibrary{arrows}
\ifdefined\isarxiv
\usepackage[margin=1in]{geometry}
\usepackage{hyperref}
\definecolor{mydarkblue}{rgb}{0,0.08,0.45}
\hypersetup{colorlinks=true, citecolor=mydarkblue,linkcolor=mydarkblue}
\else 
\usepackage{hyperref}
\definecolor{mydarkblue}{rgb}{0,0.08,0.45}
\hypersetup{colorlinks=true, citecolor=mydarkblue,linkcolor=mydarkblue}
\fi
\graphicspath{{./figs/}}
\usepackage{mathtools}

\theoremstyle{plain}
\newtheorem{theorem}{Theorem}[section]
\newtheorem{lemma}[theorem]{Lemma}
\newtheorem{definition}[theorem]{Definition}

\newtheorem{remark}[theorem]{Remark}

\newtheorem{problem}[theorem]{Problem}

\newcommand{\wt}{\widetilde}

\newcommand{\N}{\mathcal{N}}
\newcommand{\R}{\mathbb{R}}

\renewcommand{\varepsilon}{\epsilon}
\renewcommand{\tilde}{\wt}

\renewcommand{\bar}{\overline}

\DeclareMathOperator{\OPT}{\mathsf{opt}}
\DeclareMathOperator{\supp}{supp}

\DeclareMathOperator{\poly}{poly}

\DeclareMathOperator{\dist}{dist}

\DeclareMathOperator{\vect}{vec}
\DeclareMathOperator{\tr}{tr}

\newcommand{\M}{{\bf{M}}}
\newcommand{\W}{{\bf{W}}}
\newcommand{\U}{{\bf{U}}}

\newcommand{\pub}{{\mathsf{pub}}}
\newcommand{\priv}{{\mathsf{priv}}}
\newcommand{\X}{\mathbf{X}}
\newcommand{\Y}{\mathbf{Y}}
\newcommand{\calD}{\mathcal{D}}
\newcommand{\calN}{\mathcal{N}}
\renewcommand{\dist}{\mathrm{dist}}
\newcommand{\Id}{\mathsf{Id}}
\newcommand{\ETH}{$\mathsf{ETH}$}
\newcommand{\NP}{$\mathsf{NP}$}
\newcommand{\MAXCUT}{$\mathsf{MAX}$-$\mathsf{CUT}$}
\newcommand{\SAT}{$\mathsf{3SAT}$}

\newcommand{\instahide}{$\mathsf{InstaHide}$}

\definecolor{b2}{RGB}{51,153,255}
\definecolor{mygreen}{RGB}{80,180,0}
\definecolor{mypink}{rgb}{0.858, 0.188, 0.478}

\makeatletter
\newcommand*{\RN}[1]{\expandafter\@slowromancap\romannumeral #1@}
\makeatother

\usepackage{mdframed}

\newmdtheoremenv{prob}{Problem}

\ifdefined\isarxiv

\else

\usepackage[textsize=tiny]{todonotes}

\icmltitlerunning{InstaHide's Sample Complexity When Mixing Two Private Images}

\fi

\begin{document}

\ifdefined\isarxiv

\title{InstaHide's Sample Complexity When Mixing Two Private Images}
\author{
Baihe Huang\thanks{\texttt{baihe\_huang@berkeley.edu}. University of California, Berkeley.}
\and
Zhao Song\thanks{\texttt{zsong@adobe.com}. Adobe Research.}
\and
Runzhou Tao\thanks{\texttt{runzhou.tao@columbia.edu}. Columbia University.}
\and
Junze Yin\thanks{\texttt{junze@bu.edu}. Boston University.}
\and 
Ruizhe Zhang\thanks{\texttt{rzzhang@berkeley.edu}. Simons Institute for the Theory of Computing.}
\and
Danyang Zhuo\thanks{\texttt{danyang@cs.duke.edu}. Duke University.}
}
\date{}

\else

\twocolumn[
\icmltitle{InstaHide's Sample Complexity When Mixing Two Private Images}

\icmlsetsymbol{equal}{*}

\begin{icmlauthorlist}
\icmlauthor{Firstname1 Lastname1}{equal,yyy}
\icmlauthor{Firstname2 Lastname2}{equal,yyy,comp}
\icmlauthor{Firstname3 Lastname3}{comp}
\icmlauthor{Firstname4 Lastname4}{sch}
\icmlauthor{Firstname5 Lastname5}{yyy}
\icmlauthor{Firstname6 Lastname6}{sch,yyy,comp}
\icmlauthor{Firstname7 Lastname7}{comp}
\icmlauthor{Firstname8 Lastname8}{sch}
\icmlauthor{Firstname8 Lastname8}{yyy,comp}
\end{icmlauthorlist}

\icmlaffiliation{yyy}{Department of XXX, University of YYY, Location, Country}
\icmlaffiliation{comp}{Company Name, Location, Country}
\icmlaffiliation{sch}{School of ZZZ, Institute of WWW, Location, Country}

\icmlcorrespondingauthor{Firstname1 Lastname1}{first1.last1@xxx.edu}
\icmlcorrespondingauthor{Firstname2 Lastname2}{first2.last2@www.uk}

\icmlkeywords{Machine Learning, ICML}

\vskip 0.3in
]

\printAffiliationsAndNotice{\icmlEqualContribution} %

\fi 

\ifdefined\isarxiv
  \maketitle
  \begin{abstract}
  Training neural networks usually require large numbers of sensitive training data, and how to protect the privacy of training data has thus become a critical topic in deep learning research. InstaHide is a state-of-the-art scheme to protect training data privacy with only minor effects on test accuracy, and its security has become a salient question. In this paper, we systematically study recent attacks on InstaHide and present a unified framework to understand and analyze these attacks. We find that existing attacks either do not have a provable guarantee or can only recover a single private image. On the current InstaHide challenge setup, where each InstaHide image is a mixture of two private images, we present a new algorithm to recover all the private images with a provable guarantee and optimal sample complexity. In addition, we also provide a computational hardness result on retrieving all InstaHide images. Our results demonstrate that InstaHide is not information-theoretically secure but computationally secure in the worst case, even when mixing two private images.

  \end{abstract}

\else
\begin{abstract}
Training neural networks usually require large numbers of sensitive training data, and how to protect the privacy of training data has thus become a critical topic in deep learning research. InstaHide is a state-of-the-art scheme to protect training data privacy with only minor effects on test accuracy, and its security has become a salient question. In this paper, we systematically study recent attacks on InstaHide and present a unified framework to understand and analyze these attacks. We find that existing attacks either do not have a provable guarantee or can only recover a single private image. On the current InstaHide challenge setup, where each InstaHide image is a mixture of two private images, we present a new algorithm to recover all the private images with a provable guarantee and optimal sample complexity. In addition, we also provide a computational hardness result on retrieving all InstaHide images. Our results demonstrate that InstaHide is not information-theoretically secure but computationally secure in the worst case, even when mixing two private images.

\end{abstract}
\fi

\section{Introduction}
Collaboratively training neural networks based on sensitive data is appealing in many AI applications, such as healthcare, fraud detection, and virtual assistants. How to train neural networks without compromising data confidentiality and prediction accuracy has become an important and common research goal~\cite{ss15, rtd+18, ahw+18, mmr17, konevcny2016federated} both in academia and industry.
\cite{hsla20} recently proposed an approach called {\instahide} for image classification. The key idea is to train the model on a dataset where (1) each image is a mix of $k_{\priv}$ private images and $k_{\pub}$ public images, and (2) each pixel is independently sign-flipped after mixing. {\instahide} shows promising prediction accuracy on the MNIST~\cite{d12}, CIFAR-10~\cite{cifar100}, CIFAR-100, and ImageNet datasets~\cite{imagenet}. $\mathsf{TextHide}$ \cite{hscla20} applies {\instahide}'s idea to text datasets and achieves promising results on natural language processing tasks.

To understand the security aspect of {\instahide} in realistic deployment scenarios, {\instahide} authors present an {\instahide} challenge~\cite{challenge} that involves $n_{\priv} = 100$ private images, ImageNet dataset as the public images, $m = 5000$ sample images (each image is a combination of $k_{\priv} = 2$ private images and $k_{\pub} = 4$ public images and the sign of each pixel is randomly flipped). The challenge is to recover a private image given the set of sample images.

\cite{csz20} is a theoretical work that formulates the {\instahide} attack problem as a recovery problem. It also provides an algorithm to recover a private image, assuming each private and public image is a random Gaussian image (i.e., each pixel is an i.i.d. draw from $\calN(0,1)$). The algorithm shows that $O(n_{\priv}^{k_{\priv} - 2/(k_{\priv} + 1)})$ sample images are sufficient to recover one private image. \cite{carlini_attack} provides the first practical heuristic-based attack for the {\instahide} challenge ($k_{\priv} = 2$), and it can generate images that are visually similar to the private images in the {\instahide} challenge dataset. \cite{lxw+21} provides the first heuristic-based practical attack for the {\instahide} challenge ($k_{\priv} = 2$) when data augmentation is enabled. \cite{cstzz22} studied a sub-problem of the {\instahide} attack assuming that the Gram matrix can be accessed exactly, which can be regarded as the ideal case of the empirical attacks \cite{carlini_attack,lxw+21} where they used deep neural networks to estimate the Gram matrix. Under this assumption, \cite{cstzz22} proposed a theoretical algorithm based on tensor decomposition to recover the ``significant pixels'' of the private images. 

Although several researchers consider the {\instahide} challenge broken, the current {\instahide} challenge is itself too simple, and it is unclear whether existing attacks~\cite{carlini_attack, lxw+21} can still work when we use {\instahide} to protect a large number of private images (large $n$)~\cite{sanjeev}.
This raises an important question:
\begin{center}
   {\it What's the minimal number of {\instahide} images needed to recover a private image?}
\end{center}

This question is worth considering because it is a quantitative measure of how secure {\instahide} is. With the same formulation in \cite{csz20}, we achieve a better upper bound on the number of samples needed to recover private images when $k_{\priv} = 2$. Our new algorithm can recover all the private images using only $\Omega(n_{\priv} \log ( n_{\priv} ))$ samples.\footnote{For the worst case distribution, $\Omega(n_{\priv})$ is a trivial sample complexity lower bound.} This significantly improves the state-of-the-art theoretical results \cite{csz20} that requires $n_{\priv}^{4/3}$ samples to recover a single private image. However, our running time is exponential in the number of private images ($n_{\priv}$) and polynomial in the number of public images ($n_{\pub}$), where the running time of the algorithm in \cite{csz20} is polynomial in  $n_{\priv}$ and $n_{\pub}$. In addition, we provide a four-step framework to compare our attacks with the attacks presented in \cite{carlini_attack} and \cite{csz20}. We hope our framework can inspire more efficient attacks on {\instahide}-like approaches and can guide the design of future-generation deep learning algorithms on sensitive data.

\paragraph{Contributions.}

Our contributions can be summarized in the following ways.
\begin{itemize}
    \item We propose an algorithm that recovers all private images using only $\Omega(n_{\priv} \log ( n_{\priv} ))$ samples in the recent theoretical framework of attacking {\instahide} \cite{hsla20} when mixing two private images, improving the state-of-the-art result of \cite{csz20}.
    \item We summarize the existing methods of attacking {\instahide} in a unifying framework. By examining the functionality of each step, we identify the connection of a key step with problems in graph isomorphism. We also reveal the vulnerability of the existing method to recover all private images by showing the hardness of recovering all images.
\end{itemize}

\subsection{Our result}
\cite{csz20} formulates the {\instahide} attack problem as a recovery problem that given sample access to an oracle that can generate as many as {\instahide} images you want, there are two goals: 1) sample complexity, minimize the number {\instahide} images that are being used, 2) running time, use those {\instahide} images to recover the original images as fast as possible.    

Similar to \cite{csz20}, we consider the case where private and public data vectors are Gaussians. Let $S_{\pub}$ be the set of public images with $|S_{\pub}| = n_{\pub}$, let $S_{\priv}$ denote the set of private images with $|S_{\priv}| = n_{\priv}$. The model that produces {\instahide} image can be described as follows: 
\begin{itemize}
    \item Pick $k_{\pub}$ vectors from public data set  and $k_{\priv}$ vectors from private data set.
    \item Normalize $k_{\pub}$ vectors by $1/\sqrt{k_{\pub}}$ and normalize $k_{\priv}$ vectors by $1/\sqrt{k_{\priv}}$.
    \item Add $k_{\pub} + k_{\priv}$ vectors together to obtain a new vector, then flip each coordinate of that new vector independently with probability $1/2$. 
\end{itemize}
We state our results as follows:
\begin{theorem}[Informal version of Theorem~\ref{thm:main_formal}]\label{thm:main_informal}
    Let $k_{\priv} =2$. If there are $n_{\priv}$ private vectors and $n_{{\pub}}$ public vectors, each of which is an i.i.d. draw from $\calN(0,\Id_d)$, then as long as 
    \begin{align*}
    d = \Omega(\poly(k_{\pub}) \log (n_{{\pub}} + n_{\priv}) ),
    \end{align*}
    there is some $m = O(n_{\priv} \log n_{\priv} )$ such that, given a sample of $m$ random synthetic vectors independently generated as above, one can exactly recover all the private vectors in time \begin{align*}
   O(d m^2 + d n_{\pub}^2 + n_{\pub}^{2\omega+1} + m n_{\pub}^2 ) + d 2^{O(m)}
    \end{align*}
    with high probability.%
\end{theorem}
\subsection{Related work}

\paragraph{Privacy Preservation for Machine Learning.}

Privacy preservation is an important research area in machine learning. We need to train models on sensitive data, such as healthcare and finance. In these domains, protecting the privacy of training data becomes critical. One typical approach is to use differential privacy  \cite{dmns06, cms11, acg+16}. Specifically, in \cite{dmns06}, privacy is proved to be preserved through calibrating the standard deviation of the noise based on the sensitivity of the general function $f$. To protect privacy, the true answer, namely the result of $f$ on the database, is perturbed by adding random noise and then returned to the user. After that, \cite{cms11} applies this perturbation from \cite{dmns06} to empirical risk minimization. A new method, objective perturbation, is developed to preserve the privacy of the design of machine learning algorithms, where the objective function is perturbed before optimizing over classifiers. Under the differential privacy framework, \cite{acg+16} provides a refined analysis of the privacy costs and proposes a new algorithm for learning. However, using differential privacy in deep learning typically leads to substantially reduced model utility. \cite{ss15} developed a collaborative learning framework that enables training across multiple parties without exposing private data. Nevertheless, it requires extensive multi-party computation. Generative adversarial networks have recently been explored for privacy as well. \cite{xlw+18} proposed a differentially private Generative Adversarial Network (DPGAN) model. They provide adversarial training to prevent membership inference attacks. Privacy guarantees are supported by both their theoretical works and empirical evidence.

Federated learning is also an approach to preserve the privacy of machine learning. It distributes the training data and performs model aggregation via periodic parameter exchanges \cite{mmr17}. Recent works focus on three distinct aspects: communication expenses \cite{iru+19}, variations in data \cite{ag20}, and client resilience \cite{gkn17}.

\paragraph{InstaHide.} {\instahide} \cite{hsla20} proposes a different approach to achieve the preservation of privacy. The key idea is to encode the training data and perform machine learning training directly on the encoded data. InstaHide demonstrates that for images, one such encoding scheme is to mix together multiple images and randomly flip pixel signs. This encoding only has a minor effect on the utility of the model. $\mathsf{TextHide}$ \cite{hscla20} extends this idea to natural language processing tasks.
The remaining question is how secure {\instahide} is, and this has become a heavily discussed topic. \cite{csz20} first formulates the theoretical sample complexity problem for {\instahide}. \cite{carlini_attack} presents the first empirical attack on the {\instahide} challenge. 
\cite{hgssl21} evaluates {\instahide} for gradient inversion attacks in the federated learning setting.
\cite{lxw+21} provides a practical attack for the {\instahide} challenge ($k_{\priv} = 2$) when there is data augmentation using a fusion-denoising network. \cite{xh21} provides a reconstruction attack for $\mathsf{TextHide}$. %

\paragraph{Line Graph Reconstruction.} In combinatorics and graph theory, there are many works \cite{rou73, l74, lov77, prt81,s82,w92,lt93,ds95,jkl97,z97,ltv15} on reconstructing graphs (or hypergraphs) from their \emph{line graphs}, which turns out to be equivalent to the third step in our framework: assigning encoded images to original images. Each vertex in the line graph corresponds to a synthetic image, and two images are connected if the sets of private images that give rise to them overlap. When $k_{\priv}=2$ (the graph case), Whitney’s isomorphism theorem \cite{w92} characterizes which graph can be uniquely identified by its line graph, and many efficient algorithms have been proposed \cite{rou73, l74, s82,ds95,ltv15}. When $k_{\priv}>2$ (the hypergraph case), line graph reconstruction turns out to be {\NP}-hard \cite{lov77,prt81,lt93,jkl97}.    %
\paragraph{Organizations.}

In Section~\ref{sec:preli} we formulate our attack problem. In Section~\ref{sec:recover-alg} we present our algorithm and the main results. In Section~\ref{sec:conclusion} we conclude our paper and discuss future directions.

\paragraph{Notations.}
For any positive integer $n$, we use $[n]$ to denote the set $\{1,2,\cdots,n\}$.  For a set $S$, we use $\supp(S)$ to denote the support of $S$, i.e., the indices of its elements. We also use $\supp(w)$ to denote the support of vector $w\in \R^n$, i.e. the indices of its non-zero coordinates. For a vector $x$, we use $\| x \|_2$ to denotes entry-wise $\ell_2$ norm. For two vectors $a$ and $b$, we use $a \circ b$ to denote a vector where $i$-th entry is $a_i b_i$. For a vector $a$, we use $|a|$ to denote a vector where the $i$-th entry is $|a_i|$. Given a vector $v \in \R^n$ and a subset $S \subset [n]$ we use $[v]_S \in \R^{|S|}$ to denote the restriction of $v$ to the coordinates indexed by $S$.

\section{Preliminaries}
\label{sec:preli}
We use the same setup as \cite{csz20}, which is stated below.

\begin{definition}[Image matrix notation, Definition 2.2 in \cite{csz20}]
Let the image matrix $\X \in \R^{d \times n}$ be a matrix whose columns consist of vectors $x_1,\dots, x_n \in \R^d$ corresponding to $n$ images, each with $d$ pixels taking values in $\mathbb{R}$. It will also be convenient to refer to the rows of $\X$ as $p_1,\dots, p_d \in \R^n$.
\label{def:data}
\end{definition}

We define the public set, private set, and synthetic images following the setup in \cite{hsla20}.
\begin{definition}[Public/private notation, Definition 2.3 in \cite{csz20}]
We will refer to $S_{\pub} \subset [n]$ and $S_{\priv} = [n] \backslash S_{\pub}$ as the set of public and private images respectively, and given a vector $w \in \R^n$, we will refer to $\supp(w) \cap S_{\pub}$ and $\supp(w) \cap S_{\priv}$ as the public and private coordinates of $w$ respectively.
\end{definition}
\begin{definition}[Synthetic images, Definition 2.4 in \cite{csz20}]
Given sparsity levels 
\begin{align*} 
k_{\pub} \leq | S_{\pub} |, k_{\priv} \leq |S_{\priv}|,
\end{align*}
image matrix $\X \in \R^{d \times n}$ and a selection vector $w \in \R^n$ for which $[w]_{S_{\pub}}$ and $[w]_{S_{\priv}}$ are $k_{\pub}$- and $k_{\priv}$-sparse respectively, the corresponding synthetic image is the vector $y^{\X,w} = |\X w| \in \R^d$ where $|\cdot|$ denotes entrywise absolute value. We say that $\X \in \R^{d \times n}$ and a sequence of selection vectors $w_1, \dots, w_m \in \R^n$ give rise to a synthetic dataset $\Y \in \R^{m \times d}$ consisting of the images 
\begin{align*} 
(y^{\X,w_1} ,\dots, y^{\X,w_m})^\top.
\end{align*}
\end{definition}

We consider a Gaussian image, which is a common setting in phase retrieval \cite{csv13,njs13,cls15}.
\begin{definition}[Gaussian images, Definition 2.5 in \cite{csz20}]
We say that $\X$ is a random \emph{Gaussian image matrix} if its entries are sampled i.i.d. from $\calN(0,1)$.
\end{definition}

Distribution over selection vectors follows from variants of Mixup \cite{zcdlp17}. Here we $\ell_2$ normalize all vectors for convenience of analysis. Since $k_\priv$ is a small constant, our analysis can be easily generalized to other normalizations.
\begin{definition}[Distribution over selection vectors, Definition 2.6 in \cite{csz20}]
Let $\calD$ be the distribution over selection vectors defined as follows. To sample once from $\calD$, draw random subset $T_1\subset S_{\pub} , T_2\subseteq S_{\priv}$ of size $k_{\pub}$ and $k_{\priv}$ and output the unit vector whose $i$-th entry is $1 / \sqrt{k_{\pub}} $ if $i \in T_1$, $1 / \sqrt{k_{\priv}}$ if $i \in T_2$, and zero otherwise.\footnote{Note that any such vector does not specify a convex combination, but this choice of normalization is just to make some of the analysis later on somewhat cleaner, and our results would still hold if we chose the vectors in the support of $\calD$ to have entries summing to 1.}
\end{definition}

For convenience we define $\pub$ and $\priv$ operators below,
\begin{definition}[Public/private operators]
We define function $\pub(\cdot)$ and $\priv(\cdot)$ such that for vector $w \in \R^n$, $\pub(w) \in  \R^{n_\pub}$ will be the vector which only contains the coordinates of $w$ corresponding to the public subset $S_\pub$, and $\priv(w) \in  \R^{n_\priv}$ will be the vector which only contains the coordinates of $w$ corresponding to the private subset $S_\priv$. 
\end{definition}

For subset $\Tilde{S} \subset S$ we will refer to $\vect(\Tilde{S}) \in \R^n$ as the vector that
\begin{itemize} 
\item $\vect(\Tilde{S})_i = 1$ if $i \in \Tilde{S}$
\item and $\vect(\Tilde{S})_i = 0$ otherwise. 
\end{itemize}

We define the public and private components of $\W$ and $\Y$ for convenience.
\begin{definition}[Public and private components of image matrix and selection vectors]
For a sequence of selection vectors $w_1, \dots, w_m \in \R^n$ we will refer to 
\begin{align*}
\W = 
(w_1, \dots, w_m)^\top \in \R^{m \times n}
\end{align*}
as the mixup matrix. 

Specifically, we refer to $ \W_{\pub}  \in \{0,1\}^{m \times n_{\pub}}$ as the public component of mixup matrix and $\W_{\priv} \in \{0,1\}^{m \times n_{\priv}}$
as the private component of mixup matrix, i.e.,
\begin{align*}
    \W_{\pub} = \sqrt{k_\pub} \cdot  \begin{bmatrix}\pub(\W_{1,*}) \\ \vdots \\ \pub(\W_{m,*}) \end{bmatrix} \in \{0,1\}^{m \times n_{\pub}}, 
\end{align*} 
and
\begin{align*}
    \W_{\priv} = \sqrt{k_\priv} \cdot  \begin{bmatrix} \priv(\W_{1,*}) \\ \vdots \\ \priv(\W_{m,*}) \end{bmatrix}\in \{0,1\}^{m \times n_{\priv}}.
\end{align*}

We refer to $\X_{\pub} \in \R^{d \times n_\pub}$ as public component of image matrix which only contains the columns of $\X \in \R^{d \times n}$ corresponding to the public subset $S_\pub$, and $\X_{\priv} \in \R^{d \times n_\priv}$ as private component of image matrix which only contains the columns of $\X \in \R^{d \times n}$ corresponding to the private subset $S_\priv$.

Furthermore we define $\Y_{\pub} \in \R^{m \times d} $ as public contribution to {\instahide} images and $\Y_{\priv} \in \R^{m \times d}$ as private contribution to {\instahide} images:
\begin{align*}
    \Y_\pub = ~ \frac{1}{\sqrt{k_\pub}} \W_\pub \X_\pub^\top, ~~~~
    \Y_\priv = ~ \frac{1}{\sqrt{k_\priv}} \W_\priv \X_\priv^\top.
\end{align*}
\end{definition}

Instead of considering only one private image recovery as \cite{csz20}, here we consider a more complicated question that requires restoring all private images. %
\ifdefined\isarxiv
\vspace{0.3em}
\fi
\begin{samepage}
\begin{prob}[Exact Private image recovery]
\ifdefined\isarxiv
\vspace{-10pt}
\fi
Let $\X \in \R^{d \times n}$ be a Gaussian image matrix. Given access to public images $\{x_s\}_{s \in S_{\pub}}$ and the synthetic dataset $(y^{\X,w_1} ,\dots, y^{\X,w_m})$, where $w_1,\dots,w_m \sim \mathcal{D}$ are unknown selection vectors, output a set of vectors $\{\wt{x}_s\}_{s\in S_{\priv}}$ for which there exists a one-to-one mapping $\phi$ from $\{\wt{x}_s\}_{s\in S_{\priv}}$ to $\{x_s\}_{s\in S_{\priv}}$ satisfying $\phi(\wt{x}_s)_j = (x_s)_j$, $\forall j \in [d]$.
\end{prob}
\end{samepage}
We note that the one-to-one mapping $\phi$ is conceptual and only used to measure the performance of the recovery. The algorithm will not be able to learn this mapping.

\begin{algorithm*}[!t]\caption{Recovering All Private Images when $k_{\priv}=2$}
\label{alg:recover_all}
\begin{algorithmic}[1]
\Procedure{RecoverAll}{$\Y$} \Comment{Theorem~\ref{thm:main_formal}, Theorem~\ref{thm:main_informal}}
    \State \Comment{{\instahide} dataset $\Y = (y^{\X,w_1} ,\dots, y^{\X,w_m})^\top \in  \R^{m \times d}$}
    \State \Comment{{\color{blue}Step 1. Retrieve Gram matrix}}
    \State $\M \leftarrow \frac{1}{k_{\priv}+k_{\pub}} \cdot   \textsc{GramExtract}(\Y,\frac{1}{2(k_{\pub}+k_{\priv})})$ \label{lin:gram} \Comment{Algorithm 1 in \cite{csz20}}
    \State \Comment{{\color{blue}Step 2. Subtract Public images from Gram matrix}}
    \For{$i \in [m]$}
        \State $ S_i \leftarrow \textsc{LearnPublic}( \{(p_j)_{S_{\pub}},y_j^{\X,w_i})\}_{j \in [d]} )$\label{lin:public-coordinates} \Comment{Algorithm 2 in \cite{csz20}}
    \EndFor
    \State $\W_{\pub} \leftarrow (\pub(\vect(S_1)),\dots,\pub(\vect(S_m)))^\top$ \Comment{$\W_{\pub} \in \{0,1\}^{m \times n_{\pub}}$}
    \State $\M_{\priv} \leftarrow k_{\priv} \cdot (\M-\frac{1}{k_{\pub}}\W_{\pub}\W_{\pub}^\top)$ \label{lin:priv-Gram}%
    \State \Comment{{\color{blue}Step 3. Assign original images}}
    \State $\W_\priv \leftarrow \textsc{AssigningOriginalImages}( \M_{\priv} , n_{\priv} )$\label{lin:priv-weights} \Comment{Algorithm~\ref{alg:assign}}
    \State \Comment{{\color{blue}Step 4. Solving system of equations.}}
    \State $\Y_{\pub} = \frac{1}{\sqrt{k_{\pub}}}\W_{\pub} \X_{\pub}^\top$ \label{lin:Y-pub}\Comment{$\X_{\pub} \in \R^{d \times n_{\pub}}$, $\Y_{\pub} \in \R^{m \times d}$, $\W_{\pub} \in \{0,1\}^{m \times n_\pub}$}
    \State $\wt{X} \leftarrow \textsc{SolvingSystemofEquations}(\W_\priv,\sqrt{k_{\priv}}\Y_{\pub},\sqrt{k_{\priv}}\Y)$\label{lin:priv-images} \Comment{Algorithm~\ref{alg:solve}}
    \State \Return $\wt{X}$
\EndProcedure
\end{algorithmic}
\end{algorithm*}
\section{Recovering All Private Images When \texorpdfstring{$k_{\priv}=2$}{}}
\label{sec:recover-alg}
In this section, we prove our main algorithmic result. Our algorithm follows the high-level procedure introduced in Section~\ref{sec:comparison}. The detailed ideas are elaborated in the following subsections. We delay the proofs to Appendix~\ref{sec:main_formal_proof}.

\begin{theorem}[Main result]\label{thm:main_formal}
Let $S_{\pub} \subset [n]$, and let $n_{\pub} = | S_{\pub} |$ and $n_{\priv}= | S_{\priv} |$. Let $k_{\priv} = 2$. Let $k = k_{\priv} + k_{\pub}$. If $$d \geq \Omega \big(\poly(k) \log(n_{\pub}+
n_{\priv} ) \big)$$ and $$m \geq \Omega \big(n_{\priv} \log n_{\priv}\big),$$
then with high probability over $\X$ and the sequence of randomly chosen selection vectors $w_1,\dots,w_m \sim \mathcal{D}$, there is an algorithm which takes as input the synthetic dataset $\Y^\top = (y^{\X,w_1} ,\dots, y^{\X,w_m}) \in \R^{d \times m}$ and the columns of $\X$ indexed by $S_{\pub}$, and outputs $n_{\priv}$ images $\{\Tilde{x}_s\}_{s\in S_{\priv}}$ for which there exists one-to-one mapping $\phi$ from $\{\Tilde{x}_s\}_{s\in S_{\priv}}$ to $\{x_s\}_{s\in S_{\priv}}$ satisfying $\phi(\Tilde{x}_s)_j = (x_s)_j$ for all $j \in [d]$. Furthermore, the algorithm runs in time 
\begin{align*}
 O(m^2 d + d n_{\pub}^2 + n_{\pub}^{2\omega + 1} + m n_{\priv}^2 +   2^{m}\cdot mn_\priv^2d).
\end{align*}
\end{theorem}

Theorem~\ref{thm:main_formal} improves on \cite{csz20} in two aspects. First, we reduce the sample complexity from $n_\priv^{k_\priv-2/(k_\priv+1)}$ to $n_{\priv} \log n_{\priv}$ when $k_\priv = 2$ by formulating Step 3 (which is the bottleneck of \cite{csz20}) in Algorithm~\ref{alg:recover_all} as a combinatorial problem: line graph reconstruction (see Section~\ref{sec:assign}) which can be solved more efficiently. %

Second, we can recover all private images exactly instead of a single image as in \cite{csz20}, which is highly desirable for real-world practitioners. Furthermore, fixing all public images and multiplying any private image by $-1$ might not keep {\instahide} images unchanged. Thus, information-theoretically, we can recover all private images precisely (not only absolute values) as long as we have access to sufficient synthetic images. In fact, from the proof of Lemma~\ref{lem:solve_linear} our sample complexity suffices to achieve exact recovery. We note that if we repeatedly run \cite{csz20}'s algorithm to recover all private images, each run will require new samples, and the overall sample complexity will blow up by a factor of $n_{\priv}$.

\begin{remark}
The information-theoretic lower bound on the sample complexity of exactly recovering all private images is $\Omega(n_{\priv}\log(n_{\priv}))$ when $k_\priv=2$.%
This can be shown by a generalized coupon-collector argument that at least $n_{\priv}\log(n_{\priv})$ randomly generated synthetic images are required to contain all  $n_{\priv}$ private images with high probability. It indicates that our algorithm is essentially optimal with respect to the sample complexity.
\end{remark}

\subsection{Retrieving Gram matrix}\label{sec:retrieve}
In this section, we present the algorithm for retrieving the Gram matrix.
\begin{lemma}[Retrieve Gram matrix, \cite{csz20}]\label{lem:retrieve-Gram}
Let $n=n_{\pub} + n_{\priv}$. Suppose $d = \Omega(\log(m / \delta)/\eta^4)$. For a random Gaussian image matrix $\X \in \R^{d \times n}$ and arbitrary $w_1,\dots,w_m \in \mathbb{S}^{d-1}_{\geq 0}$, let $\Sigma^*$ be the output of
$\textsc{GramExtract}$ when we set $\eta = 1/2k$. Then with probability $1-\delta$ over the
randomness of $\X$, we have $\Sigma^* = k \cdot \W \W^{\top} \in \R^{m \times m}$. Furthermore, $\textsc{GramExtract}$ runs in time $O(m^2d)$.
\end{lemma}

We briefly describe how this is achieved. Without loss of generality, we may assume $S_{\priv}=[n]$, since once we determine the support of public images $S_\pub$, we can easily subtract their contribution. Consider a matrix $\Y \in \R^{m \times d}$ whose rows are $y^{\X, w_1},\dots, y^{\X, w_m}$. Then, it can be written as
\begin{align*}
    \mathbf{Y}=\begin{bmatrix}
    |\langle p_1, w_1\rangle| & \cdots & |\langle p_d, w_1\rangle|\\
    \vdots & \ddots & \vdots\\
    |\langle p_1, w_m\rangle| & \cdots & |\langle p_d, w_m\rangle|.
    \end{bmatrix}
\end{align*}
Since $\X$ is a Gaussian matrix, we can see that each column of $\mathbf{Y}$ is the absolute value of an independent draw of $\N(0, \W\W^\top)$. We define this distribution as $\N^{\mathsf{fold}}(0, \W\W^\top)$, and it can be proved that the covariance matrix of $\N^{\mathsf{fold}}(0, \W\W^\top)$ can be directly related $\W\W^\top$. Then, the task becomes estimating the covariance matrix of $\N^{\mathsf{fold}}(0, \W\W^\top)$ from $d$ independent samples (columns of $\mathbf{Y}$), which can be done by computing the empirical estimates.

\subsection{Remove public images}\label{sec:remove}
In this section, we present the algorithm for subtracting public images from the Gram matrix.
Formally, given any synthetic image $y^{\X,w}$ we recover the entire support of $[w]_{S_{\pub}}$ (essentially $\supp([w]_{S_{\pub}})$).

\begin{lemma}[Subtract public images from Gram matrix, \cite{csz20}]\label{lem:removepublic}
Let $n= n_{\priv} + n_{\pub}$. 
For any $\delta \geq 0$, if 
\begin{align*} 
d = \Omega( \poly(k_{\pub})/\log(n/\delta) ),
\end{align*}
then with probability
at least $1-\delta$ over the randomness of $\X$, we have that the coordinates output by \textsc{LearnPublic} are exactly equal to $\supp([w]_{S_{\pub}})$. Furthermore, \textsc{LearnPublic} runs in time $$O( d n_{\pub}^2 + n_{\pub}^{2\omega+1} ),$$ where $\omega \approx 2.373$ is the exponent of matrix multiplication \cite{w12}. 
\end{lemma}

Note that this problem is closely related to the Gaussian phase retrieval problem. However, we can only access the public subset of coordinates for any image vector $p_i$. We denote these partial vectors as $\{[p_i]_{S_\pub}\}_{i\in [d]}$. The first step is to construct a matrix $\wt{\M}\in \R^{n_{\pub}\times n_{\pub}}$:
\begin{align*}
    \wt{\M}= \frac{1}{d}\sum_{i=1}^d ((y^{\X, w}_i)^2 - 1) \cdot ([p_i]_{S_\pub} [p_i]_{S_\pub}^\top - \mathbf{I}).
\end{align*}
It can be proved that when $p_i$'s are i.i.d standard Gaussian vectors, the expectation of $\wt{\M}$ is $\M=\frac{1}{2}[w]_{S_\pub}[w]_{S_\pub}^\top$. However, when $d\ll n$, $\wt{\M}$ is not a sufficiently good spectral approximation of $\M$, which means we cannot directly use the top eigenvector of $\wt{\M}$. Instead, with high probability $[w]_{S_\pub}$ can be approximated by the top eigenvector of the solution of the following semi-definite programming (SDP):
\begin{align*}
\max_{Z\succeq 0}~ \langle \wt{\M}, Z\rangle ~s.t.~\tr[Z]=1, \sum_{i,j=1}^{n_\pub} |Z_{i,j}| \leq k_\pub.
\end{align*}
Hence, the time complexity of this step is $O(dn_\pub^2+n_\pub^{2\omega+1}),$ where the first term is the time cost for constructing $\wt{\M}$ and the second term is the time cost for SDP \cite{jklps20,hjst21}.

\subsection{Assigning encoded images to original images}\label{sec:assign}
 
We are now at the position of recovering $\W_\priv \in \R^{m \times n_{\priv}}$ from private Gram matrix $\M_\priv \in \R^{m \times m}$. Recall that $\M_\priv = \W_\priv \W_\priv^\top \in \R^{m \times m}$ where $\W_\priv \in \{0,1\}^{m \times n_{\priv}}$ is the mixup matrix with column sparsity $k_{\priv}$. By recovering mixup matrix $\W$ from private Gram matrix $\M$ the attacker maps each synthetic image $y^{\X,w_i},i\in [m]$ to two original images $x_{i_1},\dots,x_{i_{k_{\priv}}}$ (to be recovered in the next step) in the private data set, where $k_{\priv} = 2$.%

On the other hand, in order to recover the original image $x_{i}$ from the private data set, the attacker needs to know precisely the set of synthetic images $y^{\X,w_i},i\in [m]$ generated by $x_{i}$. Therefore this step is crucial to recover the original private images from {\instahide} images. We provide an algorithm and certify that it outputs the private component of the mixup matrix with sample complexity $m = \Omega(n_{\priv}\log n_{\priv})$.

As noted by \cite{csz20}, the intricacy of this step lies in the fact that a family of sets may not be uniquely identifiable from the cardinality of all pairwise intersections. This problem is formally stated in the following.
\begin{samepage}
\begin{prob}[Recover sets from cardinality of pairwise intersections]
\ifdefined\isarxiv
\vspace{-10pt}
\fi
Let $S_i \subset [n], i \in [m]$ be $n$ sets with cardinality $k$. Given access to the cardinality of pairwise intersections $|S_i \cap S_j|$ for all $i,j \in [m]$, output a family of sets $\wt{S}_i \subset [n], i \in [m]$ for which there exists a one-to-one mapping $\phi$ from $\wt{S}_i, i \in [m]$ to ${S}_i, i \in [m]$ satisfying $\phi(\wt{S}_j) = {S}_j$ for all $j \in [m]$.
\end{prob}
\end{samepage}

In real-world applications, attackers may not even have access to the precise cardinality of pairwise intersections $|S_i \cap S_j|$ for all $i,j \in [m]$ due to errors in retrieving the Gram matrix and public coordinates. Instead, attackers often face a harder version of the above problem, where they only know whether $|S_i \cap S_j|$ is an empty set for $i,j \in [m]$. However, for mixing two private images, the two problems are the same.

We now provide a solution to this problem. First, we define a concept closely related to the above problem.

\begin{definition}[Distinguishable]
For matrix $\M \in \R^{m \times m}$, we say $\M$ is distinguishable if there exists unique solution $\W = (w_1,\dots,w_m)^\top$ (up to permutation of rows) to the equation $\W \W^\top = \M$ such that $w_i \in \supp(\calD_{\priv})$ for all $i \in [m]$.
\end{definition}

\begin{lemma}[Assign {\instahide} images to the original images]\label{lem:assign-private-weights}
When $m = \Omega(n_{\priv}\log n_{\priv})$, let $\W_{\priv} = (w_1,\dots,w_m)^\top$ where $w_i,i \in [m]$ are sampled from distribution $\calD_{\priv}$ and $\M_{\priv} = \W_{\priv} \W_{\priv}^\top \in \R^{m \times m}$. Then with high probability $\M_{\priv}$ is distinguishable and algorithm \textsc{AssigningOriginalImages} inputs private Gram matrix $\M_{\priv} \in \{0,1,2\}^{m \times m}$ correctly outputs $\W_{\priv} \in \{0,1\}^{m \times n_{\priv}}$ with row sparsity $k_{\priv} = 2$ such that $ \W_{\priv} \W_{\priv}^{\top} = \M_{\priv}$. Furthermore \textsc{AssigningOriginalImages} runs in time $O(m n_\priv)$. %
\label{thm:assignoriginal}
\end{lemma}

The proof of Lemma~\ref{lem:assign-private-weights} is deferred to Appendix~\ref{sec:proof-assign-private-weights}. 
We consider graph $G = (V,E), |V| = n_{\priv} \text{ and } |E| = m$ where each $v_i \in V$ corresponds to an original image in private data set and each $e = (v_i,v_j) \in E$ correspond to an encrypted image generated from two original images corresponding to $v_i$ and $v_j$. We define the Gram matrix of graph $G = (V,E)$, denoted by  $\M_G \in \{0,1,2\}^{m \times m}$ where $m = |E|$, to be 
$
\M_G = \W \W^\top-\mathbf{I}
$
where $\W \in \{0,1\}^{m \times n_{\priv}}$ is the incidence matrix of G. That is \footnote{With high probability, $W$ will not have multi-edge. So, most entries of $M$ will be in $\{0,1\}$.} 
\begin{align*}
    {\M_G} = \begin{bmatrix}
    |e_1\cap e_1| & \cdots & |e_1\cap e_{m}|\\
    \vdots & \ddots & \vdots\\
    |e_{m}\cap e_1| & \cdots & |e_{m}\cap e_{m}|
    \end{bmatrix} \in \{0,1,2\}^{m \times m}.
\end{align*}

We can see that $\M_G$ actually corresponds to the line graph $L(G)$ of the graph $G$. We similarly call a graph $G$ distinguishable if there exists no other graph $G'$ such that $G$ and $G'$ have the same Gram matrix (up to permutations of edges), namely $\M_G = \M_{G'}$ (for some ordering of edges). To put it in another word, if we know $\M_G$, we can recover $G$ uniquely. Therefore, recovering $\W$ from $\M$ can be viewed as recovering graph $G$ from its Gram matrix $\M_G \in \R^{m \times m}$, and a graph is distinguishable if and only if its Gram matrix $\M_G$ is distinguishable.

This problem has been studied since the 1970s and fully resolved by Whitney \cite{w92}.
In fact, from a line graph $L(G)$ one can first identify a tree of the original graph $G$ and then proceed to recover the whole graph. The proof is then completed from well-known facts in random graph theory \cite{er60} that $G$ is connected with high probability when $m = \Omega(n_{\priv}\log n_{\priv})$. This paradigm can potentially be extended to handle $k \geq 3$ case with more information of $G$. Intuitively, this is achievable for a dense subgraph of $G$, such as the local structure identified by \cite{csz20}. It can also be achieved via a sparse Boolean matrix factorization by~\cite{cstzz22}. More discussion can be found in Appendix~\ref{sec:proof-assign-private-weights}.

\begin{algorithm*}[!t]\caption{Assigning Original Images}
\label{alg:assign}
\begin{algorithmic}[1]
\Procedure{AssigningOriginalImages}{$\M_{\priv}, n_{\priv}$}
    \State \Comment{$\M_{\priv} \in \R^{m \times m}$ is Private Gram matrix, $n_\priv$ is the number of private images}
    \State $\M_G\leftarrow \M_{\priv} - \mathbf{I}$
    \If{$n_\priv<5$}
        \For{$H\in \{0,1\}^{n_\priv \times n_\priv}$}
            \State $\M_H\leftarrow $ adjacency matrix of the line graph of $H$
            \If{$\M_H= \M_G$}
                \State $\wt{\W}\leftarrow \wt{\W}\cup \{\W_H\}$ \Comment{$\W_H$ is the incidence matrix of $H$}
            \EndIf
        \EndFor
        \State \Return $\wt{\W}$
    \EndIf
    \State Reconstruct $G$ from $\M_G$ \Comment{By Theorem~\ref{thm:line_graph_reconstruct}}
    \State \Return $\W$\Comment{The incidence matrix of $G$}
\EndProcedure
\end{algorithmic}
\end{algorithm*}

\subsection{Solving a large system of equations}\label{sec:solve}

In this section, we solve Step 4, recovering all private images by solving an $\ell_2$-regression problem. Formally, given the mixup coefficients $\W_{\priv}$ (for private images) and contributions to {\instahide} images from public images $\Y_{\pub}$ we recover all private images $\X_{\priv}$ (up to absolute value).

\begin{lemma} [Solve $\ell_2$-regression with hidden signs]\label{lem:solve_linear}
Given $\W_{\priv} \in \R^{ m \times n_{ \priv } } $ and $\Y_{\pub},\Y \in \R^{ m \times d }$. For each $i \in [d]$, let $\Y_{*,i} \in \R^m$ denote the $i$-th column of $\Y$ and similarily for ${\Y_{\pub}}_{*,i}$, the following $\ell_2$ regression
    \begin{align*}
        \min_{z_i \in \R^{n_{\priv}}} \| | \W_{\priv} z_i +{\Y_{\pub}}_{*,i}| - \Y_{*,i} \|_2 .
    \end{align*}
for all $i\in [d]$ can be solve by \textsc{SolvingSystemOfEquations} in time $O(2^{m}\cdot mn_\priv^2\cdot d).$ 
\end{lemma}
Our algorithm for solving the regression problem is given in Algorithm~\ref{alg:solve}, and the proof is deferred to Appendix~\ref{sec:solve_linear_proof}.

\begin{algorithm*}[!ht]\caption{Solving a large system of equations}
\label{alg:solve}
\begin{algorithmic}[1]
\Procedure{SolvingSystemOfEquations}{$\W_{\priv},\Y_{\pub},\Y$}
    \State \Comment{$\W_{\priv} \in \R^{m \times n_{\priv}}, \Y_{\pub}\in \R^{m \times d},\Y \in \R^{m \times d}$}
    \For{$i = 1 \to d$}
        \State $\wt{x}_i \leftarrow \emptyset$
        \For{$\sigma \in \{ -1,+1\}^m$}
            \State $z \leftarrow \min_{z \in \R^{n_{\priv} }} \| \W_{\priv} z  + {\Y_{\pub}}_{*,i} - \sigma \circ \Y_{*,i}  \|_2$
            \If{$\mathrm{sign}(\W_{\priv} z+ {\Y_{\pub}}_{*,i})=\sigma$}
                \State $\wt{x}_i \leftarrow \wt{x}_i \cup z$
            \EndIf
        \EndFor
    \EndFor
    \State $\wt{X} \leftarrow \{ \wt{x}_1, \cdots, \wt{x}_d  \}$
    \State \Return $\wt{X}$
\EndProcedure
\end{algorithmic}
\end{algorithm*}

\paragraph{Computational hardness result.}We also show that $\ell_2$-regression with hidden signs is in fact a very hard problem. Although empirical methods may bypass this issue by directly applying gradient descent, real-world practitioners taking shortcuts would certainly suffer from a lack of apriori theoretical guarantees when facing a large private dataset.

\begin{theorem}[Lower bound of $\ell_2$-regression with hidden signs, informal version of Theorem~\ref{thm:reg_lower_bound}.]
There exists a constant $\epsilon>0$ such that it is {\NP}-hard to $(1+\epsilon)$-approximate 
\begin{align*} 
\min_{z\in \R^n} \| |{\bf W} z| - y \|_2,
\end{align*} 
where ${\bf W} \in \{0,1\}^{m \times n}$ is row 2-sparse and $y \in \{0,1\}^m$.
\end{theorem}

We will reduce the {\MAXCUT} problem to the $\ell_2$-regression. {\MAXCUT} is a well-known {\NP}-hard problem \cite{bk99}. A {\MAXCUT} instance is a graph $G=(V,E)$ with $n$ vertices and $m$ edges. The goal is to find a subset of vertices $S \subseteq V$ such that the number of edges between $S$ and $V \backslash S$ is maximized, i.e., $\max_{S \subseteq V} | E ( S , V \backslash S ) |$. We can further assume $G$ is 3-regular, that is, each vertex has degree 3. 

The main idea of the proof is to carefully embed the graph into the matrix ${\bf W}$ so that if this $\ell_2$-regression can be solved with high accuracy, then an approximated {\MAXCUT} can be extracted from the solution vector $z$. Therefore, based on the {\NP}-hardness of approximating {\MAXCUT}, we can rule out the polynomial-time algorithm for $\ell_2$-regression with hidden signs. Furthermore, if we assume a fine-grained complexity-theoretic assumption (e.g., exponential time hypothesis (\textsf{ETH}) \cite{ip01}), then we can even rule out subexponential-time algorithm for $\ell_2$-regression with only constant accuracy. The full proof is deferred to Appendix~\ref{sec:hard}.

We note that our lower bound is for $\mathbf{W}$ in the worst case. It is an interesting open question to prove any average-case lower bound for this problem, i.e., can we still rule out a polynomial-time algorithm when ${\bf W}$ and $y$ are randomly sampled from some distributions?

\section{Conclusion and Future Directions}
\label{sec:conclusion}
We show that 
$\Omega(n_\priv \log n_\priv)$
samples suffice to recover all private images under the current setup for {\instahide} challenge of mixing two private images. 
We observe that a key step in attacking can be formulated as line graph reconstruction and prove the uniqueness and hardness of recovery. 
Our approach has significantly advanced the state-of-the-art approach~\cite{csz20} that requires 
$n_{\mathsf{priv}}^{4/3}$
samples to recover a single private image, and the sample complexity of our algorithm indicates that under the current setup, {\instahide} is not information-theoretically secure. On the other hand, our computational hardness result shows that {\instahide} is computationally secure in the worst case. In addition, we present a theoretical framework to reason about the similarities and differences of existing attacks~\cite{carlini_attack, csz20} and our attack on {\instahide}.

Based on our framework, there are several interesting directions for future study:
\begin{itemize}
    \item How to generalize our results to recover all private images when mixing more than two private images?
    \item How to extend this framework to analyze multi-task phase retrieval problems with real-world data?
    \item How to relax the Gaussian distribution assumption of the dataset ${\bf X}$ in this work and \cite{csz20}?
\end{itemize}

Real-world security is not a binary issue. We hope that our theoretical contributions shed light on the discussion of safety for distributed training algorithms and provide inspiration for the development of better practical privacy-preserving machine learning methods.

\ifdefined\isarxiv
\else

\section*{Impact Statement}
We believe that studying the security of {\instahide} can help develop more insights into private learning. Although our techniques give a new attacking algorithm for {\instahide} when $k_\priv=2$, it is theoretical and unlikely to be used by malicious actors. Rather, our theoretical findings show the strengths and weaknesses of {\instahide} from computation-theoretic and information-theoretic perspectives and pave the way to developing more secure and useful private learning schemes.  
\fi

\addcontentsline{toc}{section}{References}
\ifdefined\isarxiv
\bibliographystyle{alpha}
\else
\bibliographystyle{icml2024}%

\fi
\bibliography{ref}
\ifdefined\isarxiv
\else
\fi
\newpage
\appendix 
\onecolumn
\section*{Appendix}
\paragraph{Roadmap.} In Section~\ref{sec:comparison}, we present our unified framework for the InstaHide attacks. In Section~\ref{sec:carlini}, we apply our framework to give a systematic analysis of the attack in \cite{carlini_attack}. In Section~\ref{sec:proof-assign-private-weights}, we give the missing proof of Theorem \ref{thm:assignoriginal}. In Section~\ref{sec:solve_linear_proof}, we give the missing proof of Lemma~\ref{lem:solve_linear}. In Section~\ref{sec:main_formal_proof}, we provide the missing of Theorem~\ref{thm:main_formal}. In Section~\ref{sec:hard}, we show a computational lower bound for the attacking problem.

\section{A Unified Framework to Compare With Existing Attacks}
\label{sec:comparison}

\begin{table*}[htbp]
    \centering
    \caption{\small A summary of running times in different steps between ours and \cite{csz20}. This table only compares the theoretical results. Let $k_{\priv}$ denote the number of private images we select in {\instahide} image. Let $d$ denote the dimension of the image. Let $n_{\pub}$ denote the number of images in the public dataset. Let $n_{\priv}$ denote the number of images in the private dataset. We provide a computational lower bound for Step 4 in Appendix~\ref{sec:hard}. There is no algorithm that solves Step 4 in $2^{o(n_{\priv})}$ time under Exponential Time Hypothesis ({\ETH}) (Theorem~\ref{thm:reg_lower_bound}). Let {\bf Rec.} denote the Recover.} 
    \begin{tabular}{|l|l|l|l|l|l|l|l|} \hline
         {\bf Refs} & {\bf Rec.} & $k_{\priv}$ & {\bf Samples} & {\bf Step 1} & {\bf Step 2} & {\bf Step 3} & {\bf Step 4} \\ \hline
         Chen & one & $\geq 2$ & $m \geq n^{k_\priv-2/(k_\priv+1)} $ & $d m^2$ & $d n_{\pub}^2 + n_{\pub}^{2\omega + 1}$ & $m^2$ & $2^{k_{\priv}^2}$ \\ \hline
         Ours & all & $=2$ & $m \geq n_{\priv} \log n_{\priv}$ & $d m^2$ & $d n_{\pub}^2 + n_{\pub}^{2\omega+1} $ & $mn_{\priv}$ & $2^{m}\cdot n_\priv^2d$ \\ \hline
    \end{tabular}
    \label{tab:my_label}
\end{table*}

Our attack algorithm (Algorithm~\ref{alg:recover_all}) contains four steps for $k_{\priv} = 2$. We can prove $m = O(n_{\priv}  \log (n_{\priv}) )$ suffices for exact recovery. Our algorithm shares similarities with two recent attack results: one is a practical attack \cite{carlini_attack}, and the other is a theoretical attack \cite{csz20}. In the next few paragraphs, we describe our attack algorithm in four major steps. For each step, we also give a comparison with the corresponding step in \cite{carlini_attack} and \cite{csz20}. %

\begin{itemize}
    \item {\bf Step 1.} Section~\ref{sec:retrieve}. Recover the Gram matrix $\M  = \W \W^\top \in \R^{m \times m}$ of mixup weights $\W$ from synthetic images $\Y$. This Gram matrix contains all inner products of mixup weights $\langle w_i,w_j \rangle$. Intuitively this measures the similarity of each pair of two synthetic images and is a natural start of all existing attacking algorithms.
    \begin{itemize}
        \item For this step, \cite{carlini_attack}'s attack uses a pre-trained neural network on the public dataset to construct the Gram matrix.
        \item For this step, note that $\Y$ follows folded Gaussian distribution whose covariance matrix is directly related to $\M$. We can thus solve this step by estimating the covariance of a folded Gaussian distribution. This is achieved by using  Algorithm 2 in \cite{csz20}. It takes $O(m^2 d)$ time.
    \end{itemize}
    \item {\bf Step 2.} Section~\ref{sec:remove}. Recover all public image coefficients and subtract the contribution of public coefficients from the Gram matrix $\M$ to obtain $\M_\priv$. This step is considered as one of the main computational obstacles for private image recovery.%
    \begin{itemize}
        \item For this step, \cite{carlini_attack}'s attack: 1) they treat public images as noise, 2) they don't need to take care of the public images' labels, since current {\instahide} Challenge doesn't provide a label for public images.
        \item For this step, we invoke a paradigm in sparse phase retrieval by using a general SDP solver to approximate the principle components of the Gram matrix of public coefficients. \cite{csz20} proved that this method exactly outputs all public coefficients. The time of this step has two parts: 1) formulating the matrix, which takes $d n_{\pub}^2$, 2) solving an SDP with $n_{\pub}^2 \times n_{\pub}^2$ size matrix variable and $O(n_{\pub}^2)$ constraints, which takes $n_{\pub}^{2\omega+1}$ time \cite{jklps20,hjst21}, where $\omega$ is the exponent of matrix multiplication.
    \end{itemize}
    \item {\bf Step 3.} Section~\ref{sec:assign}. Recover private coefficients $\W_\priv \in \R^{m \times n_{\priv} }$ from private Gram matrix $\M_\priv$ ($\M_\priv = \W_\priv \W_\priv^\top$), this step takes $O(m \cdot n_{\priv}^2)$ time.
    \begin{itemize}
        \item For this step, \cite{carlini_attack}'s attack uses $K$-means to figure out cliques and then solves a min-cost max flow problem to find the correspondence between {\instahide} image and original image (see Appendix~\ref{sec:carlini} for detailed discussions).
        \item For this step \cite{csz20} starts by finding a local structure called ``floral matrix'' in the Gram matrix. They prove the existence of this local structure when $m \geq n_\priv^{k_\priv-2/(k_\priv+1)}$. Then \cite{csz20} can recover private coefficients within that local structure using nice combinatorial properties of the ``floral matrix''.
        \item For this step, we note the fact that in $k_{\priv}=2$ case the mixup matrix corresponds to the incident matrix of a graph $G$ and the Gram matrix corresponds to its line graph $L(G)$ (while $k_{\priv}\geq 3$ cases correspond to hypergraphs). We can then leverage results in graph isomorphism theory to recover all private coefficients. In particular, when $m \geq \Omega(n_\priv \log n_\priv)$ the private coefficients are uniquely identifiable from the Gram matrix.
    \end{itemize}
    \item {\bf Step 4.} Section~\ref{sec:solve}. Solve $d$ independent $\ell_2$-regression problems to find private images $X_\priv$. Given $\W_\priv \in \R^{ m \times n_{ \priv } } $ and $\Y \in \R^{ m \times d }$. For each $i \in [d]$, let $\Y_{*,i} \in \R^m$ denote the $i$-th column of $\Y$, we need to solve the following $\ell_2$ regression
    \begin{align*}
        \min_{z \in \R^{n_{\priv}}} \| | \W_\priv z + {\Y_\pub}_{*,i} | - | \Y_{*,i} | \|_2 .
    \end{align*}
    The classical $\ell_2$ regression can be solved in an efficient way in both theory and practice. However, here we don't know the random signs and we have to consider all $2^m$ possibilities. In fact, we show that solving $\ell_2$ regression with hidden signs is NP-hard.
    \begin{itemize}
        \item For this step, \cite{carlini_attack}'s attack is a heuristic algorithm that uses gradient descent. 
        \item For this step, we enumerate all possibilities of random signs to reduce it to standard $\ell_2$ regressions. \cite{csz20}'s attack is doing the exact same thing as us. However, since their goal is just recovering one private image (which means $m=O(k^2)$) they only need to guess $2^{k^2}$ possibilities.
    \end{itemize}
\end{itemize}
\section{Summary of the Attack by  \texorpdfstring{\cite{carlini_attack}}{} }
\label{sec:carlini}
This section summarizes the result of Carlini et al, which is an attack of {\instahide} when $k_{\priv} = 2$. \cite{carlini_attack}.\
We first briefly describe the current version of {\instahide} Challenge. Suppose there are $n_{\priv}$ private images, the {\instahide} authors \cite{hsla20} first choose a parameter $T$, this can be viewed as the number of iterations in the deep learning training process. For each $t \in [T]$, \cite{hsla20} draws a random permutation $\pi_t : [n_{\priv}] \rightarrow [n_{\priv}]$. Each {\instahide} image is constructed from a private image $i$, another private image $\pi_t(i)$ and also some public images. Therefore, there are $T \cdot n_{\priv}$  {\instahide} images in total. Here is a trivial observation: each private image shown up in exactly $2T$ {\instahide} images (because $k_{\priv} = 2$). The model in \cite{csz20} is a different one: each {\instahide} image is constructed from two random private images and some random public images. Thus, the observation that each private image appears exactly $2T$ does not hold. In the current version of {\instahide} Challenge, the {\instahide} authors create the {\instahide} labels (a vector that lies in $\R^L$ where the $L$ is the number of classes in image classification task) in a way that the label of an {\instahide} image is a linear combination of labels (i.e., one-hot vectors) of the private images and not the public images. This is also a major difference compared with \cite{csz20}. Note that \cite{carlini_attack} won't be confused about, for the label of an {\instahide} image, which coordinates of the label vector are from the private images and which are from the public images.

\begin{itemize}

    \item {\bf Step 1.} Recover a similarity\footnote{In \cite{carlini_attack}, they call it similarity matrix, in \cite{csz20} they call it Gram matrix. Here, we follow \cite{carlini_attack} for convenience.} matrix $\M \in \{0,1,2\}^{m \times m}$. %
    \begin{itemize}
        \item Train a deep neural network based on all the public images, and use that neural network to construct the similarity matrix $\M$.
    \end{itemize}
    \item {\bf Step 2.} Treat public image as noise.

    \item {\bf Step 3. Clustering.}
        This step is divided into 3 substeps. 
        
        The first substep uses the similarity matrix $\M$ to
        construct $T n_\priv$ clusters of {\instahide} images based on each {\instahide} image such that the 
        images inside one cluster shares a common original image.
        
        The second substep runs $K$-means on these clusters, to group clusters into $n_\priv$ groups such that
        each group corresponds to one original image.
        
        The third substep constructs a min-cost flow graph to compute the two original images that each {\instahide} image is mixed from.
    \begin{itemize}
        \item {\bf Grow clusters.} Figure~\ref{fig:cluster} illustrates an example of this step. For a subset $S$ of {\instahide} images ($S \subset [m]$), we define $\textsc{insert}(S)$ as
            \begin{align*}
                \textsc{Insert}(S) = S \cup \arg\max_{i \in [m]} \sum_{j \in S} \M_{i, j}
            \end{align*}
            For each $i \in [m]$, we compute set $S_i \subset [m]$ where $S_i = \textsc{insert}^{(T/2)}(\{i\})$.
        \item {\bf Select cluster representatives.} Figure~\ref{fig:cluster} illustrates an example of this step. Define distance between clusters as
            \begin{align*}
                \dist(i, j) = \frac{|S_i \cap S_j|}{|S_i \cup S_j|}.
            \end{align*}
            Run $k$-means using metric $\dist :[m] \times [m] \rightarrow \R$ and $k = n_\priv$. Result is $n_\priv$ groups $C_1, \ldots, C_{n_\priv} \subseteq [m]$.
            Randomly select a representative $r_i \in C_i$, for each $i \in [n_\priv]$.
        \item {\bf Computing assignments.} Construct a min-cost flow graph as Figure~\ref{fig:min_cost_flow}, with weight matrix $\wt{\W} \in \R^{m \times n_{\priv}}$ defined as follows:
            \begin{align*}
                \tilde{\W}_{i,j} = \frac{1}{|S_{r_j}|} \sum_{k \in S_{r_j}} \M_{i, k}.
            \end{align*}
            for $i \in [m], j \in [n_\priv]$. 
            After solving the min-cost flow (Figure~\ref{fig:min_cost_flow_result}), construct the assignment matrix $\W \in \R^{m \times n_{\priv}}$ such that
            $\W_{i,j} = 1$ if the edge from $i$ to $j$ has flow, and $0$ otherwise.
            
    \end{itemize}
        
    \item {\bf Step 4.} Recover original image.
        From Step 3, we have the unweighted assignment matrix $\W \in \{0,1\}^{m \times n_{\priv}}$. Before we recover the original image,
        we need to first recover the weight of mixing, which is represented by the weighted assignment matrix $\U \in \R^{m \times n_{\priv}}$. To recover weight, we first recover the label for each cluster group, and use the recovered label
        and the mixed label to recover the weight.
    \begin{itemize}
        \item First, we recover the label for each cluster, for all $i \in [n_\priv]$. Let $L$ denote the number of classes in the classification task of {\instahide} application.
        For $j \in [m]$, let $y_j \in \R^{L}$ be the label of $j$.
        \begin{align*}
            \mathrm{label}(i) = \bigcap_{j \in [m], \W_{j,i} = 1} \supp(y_j) \in [L].
        \end{align*}
        Then, for $i \in [m]$ and $j \in [n_\priv]$ such that $\W_{i,j} = 1$, define
            $\U_{i,j} = y_{i, \mathrm{label}(j)}$ for $|\supp(y_i)| = 2$ and $\U_{i,j} = y_{i, \mathrm{label}(j)} / 2$
            for $|\supp(y_i)| = 1$.
            
        Here, $\W \in \{0,1\}^{m \times n_{\priv}}$ is the unweighted assignment matrix and $\U \in \R^{m \times n_{\priv}}$ is the weighted assignment matrix.
        For $\W_{i,j} = 0$, let $\U_{i,j} = 0$.
        \item Second, for each pixel $i \in [d]$, we run gradient descent to find the original images.
        Let $\Y \in \R^{ m \times d }$ be the matrix of all {\instahide} images, $\Y_{*,i}$ denote the $i$-th column of $\Y$.\footnote{The description of the attack in \cite{carlini_attack} recovers original images by using gradient descent for $\min_{z \in \R^{n_{\priv}}} \|  \U | z | - |\Y_{*,i}| \|_2$, which we believe is a typo.}
        \begin{align*}
            \min_{z \in \R^{n_{\priv}}} \|  | \U z | - |\Y_{*,i}| \|_2 .
        \end{align*}
    \end{itemize}
\end{itemize}

\section{Missing Proofs for Theorem~\ref{thm:assignoriginal}}\label{sec:proof-assign-private-weights}
For simplicity, let $\W$ denote $\W_\priv$ and $\M$ denote $\M_\priv$ in this section.
\subsection{A graph problem (\texorpdfstring{$k_{\priv}=2$}{})}

\begin{theorem}[\cite{w92}]\label{thm:whitney_iso}
Suppose $G$ and $H$ are connected simple graphs and $L(G)\cong L(H)$. If $G$ and $H$ are not $K_3$ and $K_{1,3}$, then $G\cong H$. Furthermore, if $|V(G)|\geq 5$, then an isomorphism of $L(G)$ uniquely determines an isomorphism of $G$.
\end{theorem}

In other words, this theorem claims that given $\M = \W \W^\top$, if the underlying $\W$ is not the incident matrix of $K_3$ or $K_{1,3}$, $\W$ can be uniquely identified up to permutation. Theorem~\ref{thm:whitney_iso} can also be generalized to the case when $G$ has multi-edges \cite{z97}.

On the other hand, a series of work \cite{rou73, l74, s82,ds95,ltv15} showed how to efficiently reconstruct the original graph from its line graph:
\begin{theorem}[\cite{ltv15}]\label{thm:line_graph_reconstruct}
Given a graph $L$ with $m$ vertices and $t$ edges, there exists an algorithm that runs in time $O(m+t)$ to decide whether $L$ is a line graph and output the original graph $G$. Furthermore, if $L$ is promised to be the line graph of $G$, then there exists an algorithm that outputs $G$ in time $O(m)$. 
\end{theorem}

With Theorem~\ref{thm:whitney_iso} and Theorem~\ref{thm:line_graph_reconstruct}, Theorem~\ref{thm:assignoriginal} follows immediately:

\begin{proof}[Proof of Theorem \ref{thm:assignoriginal}]
First, since $m= \Omega(n_{\priv}\log (n_{\priv}))$, a well-known fact in random graph theory by Erd{\H{o}}s and R{\'e}nyi \cite{er60} showed that the graph $G$ with incidence matrix $\W$ will almost surely be connected. Then, we compute $\M_G=\M-\mathbf{I}$, the adjacency matrix of the line graph $L(G)$. Theorem~\ref{thm:whitney_iso} implies that $G$ can be uniquely recovered from $\M_G$ as long as $n_\priv$ is large enough. Finally, We can reconstruct $G$ from $\M_G$ by Theorem~\ref{thm:line_graph_reconstruct}. 

For the time complexity of Algorithm~\ref{alg:assign}, the reconstruction step can be done in $O(m)$ time. Since we need to output the matrix $\W$, we will take $O(mn_\priv)$-time to construct the adjacency matrix of $G$. Here, we do not count the time for reading the whole matrix $\M$ into memory.

\end{proof}

\subsection{General case (\texorpdfstring{$k_{\priv}>2$}{})}
The characterization of $\M$ and $\W$ as the line graph and incidence graph can be generalized to $k_{\priv}>2$ case, which corresponds to hypergraphs. 

Suppose $\M = \W  \W^\top$ with $k_{\priv}=k>2$. Then, $\W$ can be recognized as the incidence matrix of a $k$-uniform hypergraph $G$, i.e., each hyperedge contains $k$ vertices. $\M_G=\M-\mathbf{I}$ corresponds to adjacency matrix of the line graph of hypergraph $G$: (${\M_G})_{i,j}=|e_i \cap e_j|$ for $e_i,e_j$ being two hyperedges. Now, we can see that each entry of $\M_G$ is in $\{0,\dots,k\}$.

Unfortunately, the identification problem becomes very complicated for hypergraphs. Lov{\'a}sz \cite{lov77} stated the problem of characterizing the line graphs of 3-uniform hypergraphs and noted that Whitney's isomorphism theorem (Theorem~\ref{thm:whitney_iso}) cannot be generalized to hypergraphs. Hence, we may not be able to uniquely determine the underlying hypergraph and we should just consider a more basic problem: 
\begin{problem}[Line graph recognition for hypergraph]\label{prob:line_recog}
Given a simple graph $L=(V,E)$ and $k\in \mathbb{N}$,  decide if $L$ is the line graph of a $k$-uniform hypergraph $G$.
\end{problem}
 
Even for the recognition problem, it was proved to be NP-complete for fixed $k\geq 3$ \cite{lt93, prt81}. However, Problem~\ref{prob:line_recog} becomes tractable if we add more constraints to the underlying hypergraph $G$. First, suppose $G$ is a linear hypergraph, i.e., the intersection of two hyperedges is at most one. If we further assume the minimum degree of $G$ is at least 10, i.e., each vertex are in at least 10 hyperedges, there exists a polynomial-time algorithm for the decision problem. Similar result also holds for $k>3$ \cite{jkl97}. Let the edge-degree of a hyperedge be the number of triangles in the hypergraph containing that hyperedge. \cite{jkl97} showed that assuming the minimum edge-degree of $G$ is at least $2k^2-3k+1$, there exists a polynomial-time algorithm to decide whether $L$ is the line graph of a linear $k$-uniform hypergraph. Furthermore, in the yes case, the algorithm can also reconstruct the underlying hypergraph. We also note that without any constraint on minimum degree or edge-degree, the complexity of recognizing line graphs of $k$-uniform linear hypergraphs is still unknown.   

\begin{figure*}[!t]
\centering
\includegraphics[width=0.85\textwidth]{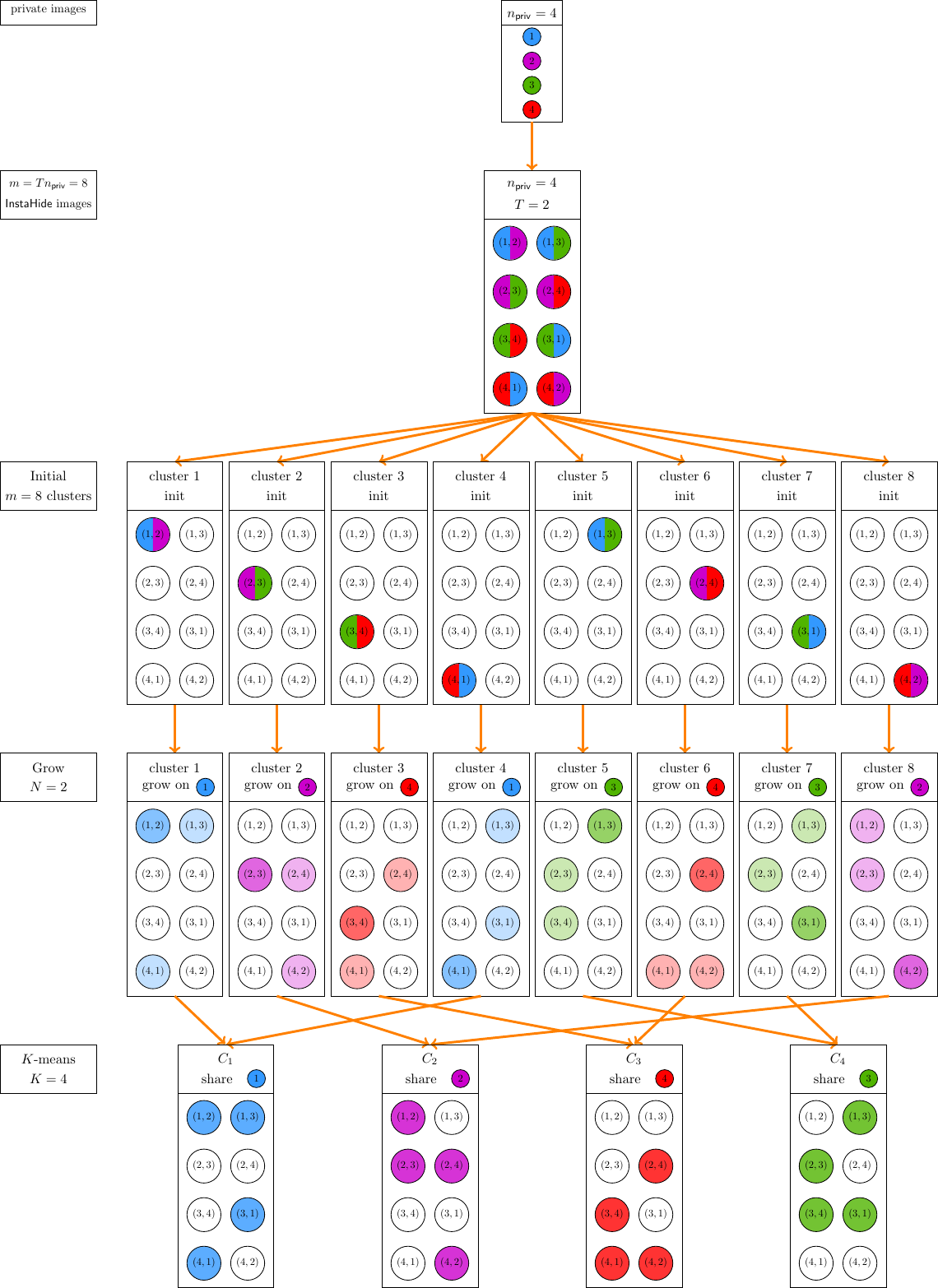}
\caption{An example about cluster step in \cite{carlini_attack} for $T = 2$ and $n_\priv = 4$. First, starting from each {\instahide} image (top), the algorithm grows cluster $S_i$ with size $3$ (middle). Then, we use $K$-means for $K = 4$ to compute 4 groups $C_1, \ldots, C_4$ (bottom), these groups each correspond to one original image.
}\label{fig:cluster}
\end{figure*}

\begin{figure*}[!t]
\centering
\includegraphics[width=0.64\textwidth]{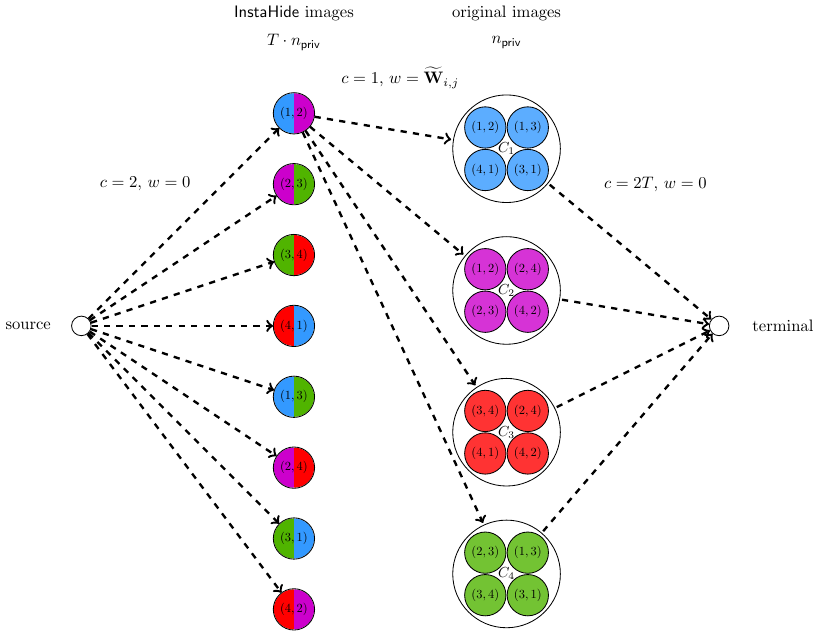}
\caption{\small The construction of the graph for min-cost max flow. $c$ denotes the flow capacity of the edge, and $w$ denote the weight of the edge. 
The graph contains $T \cdot n_{\priv} $ nodes for each {\instahide} images, $n_{\priv}$ nodes for each original images, a source and a terminal.
There are three types of edges: i) (left) from the source to each {\instahide} image node, with flow capacity $2$ and weight $0$; ii) (middle) from each {\instahide} image node $i$ to each original image node $j$, with flow capacity $1$ and weight $\wt{\W}_{i,j}$; iii) (right) from each original image node to the terminal, with flow capacity $2T$ and weight $0$.
}\label{fig:min_cost_flow}
\end{figure*}

\begin{figure*}[!t]
\centering
\includegraphics[width=0.64\textwidth]{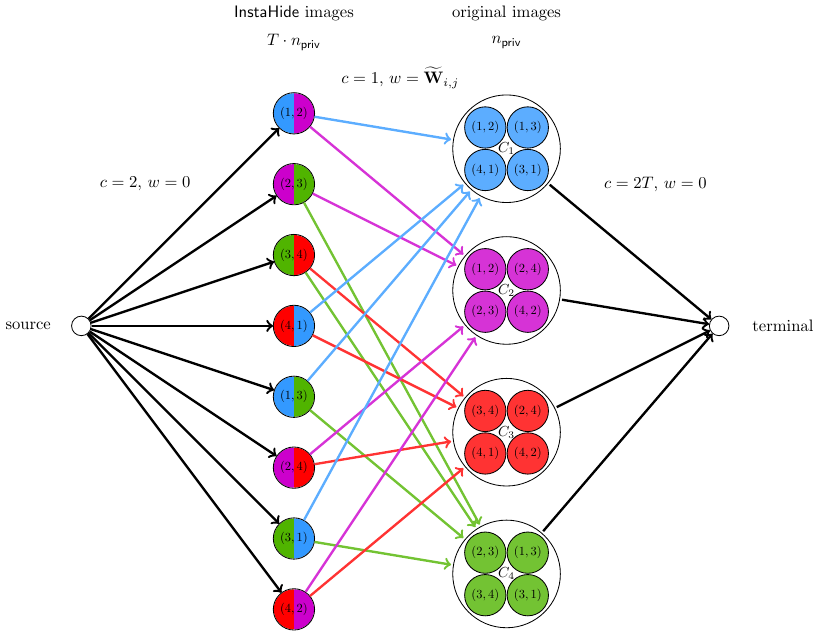}
\caption{The result of solving the min-cost flow in Figure~\ref{fig:min_cost_flow}. Each {\instahide} image is assigned to two clusters, which ideally correspond to two original images. In reality, a cluster may not contain all {\instahide} images that share the same original image.
}\label{fig:min_cost_flow_result}
\end{figure*}

\section{Missing Proof for Lemma~\ref{lem:solve_linear}}\label{sec:solve_linear_proof}
\begin{lemma} [Solve $\ell_2$-regression with hidden signs, Restatement of Lemma~\ref{lem:solve_linear}]
Given $\W_{\priv} \in \R^{ m \times n_{ \priv } } $ and $\Y_{\pub},\Y \in \R^{ m \times d }$. For each $i \in [d]$, let $\Y_{*,i} \in \R^m$ denote the $i$-th column of $\Y$ and similarily for ${\Y_{\pub}}_{*,i}$, the following $\ell_2$ regression
    \begin{align*}
        \min_{z_i \in \R^{n_{\priv}}} \| | \W_{\priv} z_i +{\Y_{\pub}}_{*,i}| - \Y_{*,i} \|_2 .
    \end{align*}
for all $i\in [d]$ can be solve by \textsc{SolvingSystemOfEquations} in time $O(2^{m}\cdot mn_\priv^2\cdot d)$. %
\end{lemma}

\begin{proof}
Suppose %
\begin{align*}
    \W_{\priv} = \begin{bmatrix}
    | & | & & |\\
    w_1 & w_2 & \cdots & w_m\\
    | & | & & |
    \end{bmatrix}^\top.
\end{align*}
and fix a coordinate $i\in [d]$.

Then, the $\ell_2$-regression we considered for the $i$-th coordinate actually minimizes 
\begin{align*}
     &~ \min_{z_i \in \R^{n_{\priv}}}\sum_{j=1}^m (|w_j^\top z_i+{\Y_{\pub}}_{j,i}|-\Y_{j,i})^2\\
    = &~ \min_{z_i \in \R^{n_{\priv}}}
    \sum_{j=1}^m (w_j^\top z_i +{\Y_{\pub}}_{j,i} - \sigma_j\cdot \Y_{j,i})^2,
\end{align*}
where $\sigma_j\in \{-1, 1\}$ is the sign of $w_j z_i^*$ for the minimizer $z_i^*$.

Therefore, in Algorithm~\ref{alg:solve}, we enumerate all possible $\sigma\in \{\pm 1\}^m$. Once $\sigma$ is fixed, the optimization problem becomes the usual $\ell_2$-regression, which can be solved in $O(n_\priv^{\omega}+mn_\priv^2)$ time. Since we assume $m=\Omega(n_\priv \log (n_\priv))$ in the previous step, the total time complexity is $$O(2^{m}\cdot mn_\priv^2).$$

If $\mathrm{sign}(\W_{\priv} z+ {\Y_{\pub}}_{*,i})=\sigma$ holds for 
\begin{align*}
    z = \min_{z \in \R^{n_{\priv} }} \| \W_{\priv} z  + {\Y_{\pub}}_{*,i} - \sigma \circ \Y_{*,i}  \|_2,
\end{align*}
then $$\sum_{j=1}^m (|w_j^\top z_i+{\Y_{\pub}}_{j,i}|-\Y_{j,i})^2 = 0$$ and $z$ is the unique minimizer of the signed $\ell_2$-regression problem almost surely.

Indeed, if we have for $\sigma \neq \wt{\sigma}$, 
$$\W_{\priv} {(X_{\priv}^{\top})}_{*,i}  + {\Y_{\pub}}_{*,i} - \sigma \circ \Y_{*,i} = 0$$ and $$\W_{\priv} \wt{z}  + {\Y_{\pub}}_{*,i} - \wt{\sigma} \circ \Y_{*,i} = 0$$ hold, then from direct calculations we come to $$\W_{\priv} \wt{z} = \sigma \circ \wt{\sigma} \circ (\W_{\priv} {(X_{\priv}^{\top})}_{*,i} + {\Y_{\pub}}_{*,i})- {\Y_{\pub}}_{*,i}.$$ 
This indicates that $$\sigma \circ \wt{\sigma} \circ (\W_{\priv} {(X_{\priv}^{\top})}_{*,i} + {\Y_{\pub}}_{*,i})- {\Y_{\pub}}_{*,i}$$ lies in a $n_\priv$-dimensional subspace of $\R^m$. Noting that ${(X_{\priv}^{\top})}_{j,i}$ and ${\Y_{\pub}}_{j,i}$ are i.i.d sampled from Gaussian, the event above happens with probability zero since $m \gg n_\priv$. Thus, we can repeat this process for all $i\in [d]$ and solve all $z_i$'s.
\end{proof}
\begin{remark}
From the above proof, we can see that Algorithm~\ref{alg:solve} also works for more general mixing methods: 
\begin{itemize}
    \item Suppose each synthetic image is generated by randomly picking 2 private images with $k_{\pub}$ public images and applying a function $\sigma_1:\R^d\rightarrow \R^d$ to the linear combination of these images. If $|\sigma_1^{-1}(x)|\leq c_1$ for all $x\in \R^d$, then Step 4 takes $O(c_1^m\cdot mn_{\priv}^2d)$ time.  
    \item Suppose the synthetic image is generated in the same way but applying a function $\sigma_2:\R\rightarrow \R$ to each coordinate of the linear combination of selected images. If $|\sigma_2^{-1}(x)|\leq c_2$ for all $x\in \R$ (in most cases $c_2\ll c_1$), then Step 4 takes $O(c_2^m\cdot mn_{\priv}^2d)$ time.
\end{itemize}

\end{remark}

\section{Missing Proof for Theorem~\ref{thm:main_formal}}\label{sec:main_formal_proof}
In this section, we provide the proof of Theorem~\ref{thm:main_formal}.
\begin{proof}
By Lemma~\ref{lem:retrieve-Gram} the matrix computed in Line~\ref{lin:gram} satisfies $\M = \W \W^\top$. By Lemma~\ref{lem:removepublic}, Line~\ref{lin:public-coordinates} correctly computes the indices of all public coordinaes of $w_i, i \in [m]$. Therefore from 
\begin{align*}
\M = & ~ \W \W^\top \\
= & ~ \W_\priv \W_\priv^\top / k_\priv + \W_\pub \W_\pub^\top / k_\pub,
\end{align*}
the Gram matrix computed in Line~\ref{lin:priv-Gram} satisfies $$\M_\priv = \W_\priv \W_\priv^\top.$$

We can now apply Lemma~\ref{lem:assign-private-weights} to find the private components of mixup weights.

Indeed, the output of Line~\ref{lin:priv-weights} is exactly
\begin{align*}
    \W_{\priv} =&~ k_\priv \cdot  \begin{bmatrix} \priv(\W_{1,*}) \\ \vdots \\ \priv(\W_{m,*}) \end{bmatrix} \in \{0,1\}^{m \times n_{\priv}}.
\end{align*}
Based on the correctness of private weights, the output in Line~\ref{lin:priv-images} is exactly all private images by Lemma~\ref{lem:solve_linear}. This completes the proof of correctness of Algorithm~\ref{alg:recover_all}.

By Lemma~\ref{lem:retrieve-Gram}, Line~\ref{lin:gram} takes in time 
\begin{align*} 
O(m^2d).
\end{align*}

By Lemma~\ref{lem:removepublic} Line~\ref{lin:public-coordinates} runs in time 
\begin{align*} 
O(dn_\pub^2+n_\pub^{2\omega+1}).
\end{align*}

By Lemma~\ref{lem:assign-private-weights} private weights can be computed in time 
\begin{align*} 
mn_\priv.
\end{align*}

Line~\ref{lin:priv-Gram} and Line~\ref{lin:Y-pub} can be computed efficiently in time 
\begin{align*} 
O(m^\omega).
\end{align*}

Finally Line~\ref{lin:priv-images} is computes in time 
\begin{align*} 
O(2^{m}\cdot mn_\priv^2\cdot d)
\end{align*}
by Lemma~\ref{lem:solve_linear}. 

Combining all these steps, the total running time of Algorithm~\ref{alg:recover_all} is bounded by $$O(m^2 d + d n_{\pub}^2 + n_{\pub}^{2\omega + 1} + m n_{\priv}^2 +   2^{m}\cdot mn_\priv^2d).$$

Thus, we complete the proof.
\end{proof}

\section{Computational Lower Bound}\label{sec:hard}
The goal of this section is to prove that the $\ell_2$-regression with hidden signs is actually a very hard problem, even for approximation (Theorem~\ref{thm:reg_lower_bound}), which implies that Algorithm~\ref{alg:solve} cannot be significantly improved. For simplicity we consider $S_{\pub} = \emptyset$. 

We first state an NP-hardness of approximation result for 3-regular {\MAXCUT}.
\begin{theorem}[Imapproximability of 3-regular {\MAXCUT}, \cite{bk99}]\label{thm:maxcut_hard}
For every $\epsilon>0$, it is {\NP}-hard to approximate 3-regular {\MAXCUT} within a factor of $r + \epsilon$, where $r \approx 0.997$.
\end{theorem}

If we assume the Exponential Time Hypothesis ({\ETH}), which a plausible assumption in theoretical computer science, we can get stronger lower bound for {\MAXCUT}. 
\begin{definition}[Exponential Time Hypothesis ({\ETH}), \cite{ip01}]
There exists a constant $\epsilon>0$ such that the time complexity of $n$-variable {\SAT} is at least $2^{\epsilon n}$.
\end{definition}

\begin{theorem}[\cite{flp16}]\label{thm:maxcut_eth_hard}
Assuming {\ETH}, there exists a constant $0<r'<1$ such that no $2^{o(n)}$-time algorithm can $r'$-approximate the MaxCut of an $n$-vertex, 5-regular graph. %
\end{theorem}

With Theorem~\ref{thm:maxcut_hard} and Theorem~\ref{thm:maxcut_eth_hard}, we can prove the following inapproximability result for the $\ell_2$-regression problem with hidden signs. 

\begin{theorem}[Lower bound of $\ell_2$-regression with hidden signs]\label{thm:reg_lower_bound}
There exists a constant $\epsilon>0$ such that it is {\NP}-hard to $(1+\epsilon)$-approximate
\begin{align}\label{eq:l2_reg}
    \min_{z\in \R^n} \| |W z| - y \|_2,
\end{align}
where $W \in \{0,1\}^{m \times n}$ is row 2-sparse and $y \in \{0,1\}^m$.

Furthermore, assuming {\ETH}, there exists a constant $\epsilon'$ such that no $2^{o(n)}$-time algorithm can $\epsilon'$-approximate Eq.~\eqref{eq:l2_reg}.
\end{theorem}
\begin{proof}

Given a 3-regular {\MAXCUT} instance $G$, we construct an $\ell_2$-regression instance $(W, y)$ with $W\in \{0,1\}^{m'\times n}$ and $y\in \{0,1\}^{m'}$ where $m'=m+cn=(1+3c/2)m$ and $c=10^6$ as follows. 
\begin{itemize}
    \item For each $i\in [m]$, let the $i$-th edge of $G$ be $e_i=\{u,v\}$. We set $W_{i,*}$ to be all zeros except the $u$-th and $v$-th coordinates being one. That is, we add a constraint $|z_u+z_v|$. And we set $y_i = 0$.
    \item For each $j\in [n]$, we set $W_{m+c(j-1)+1,*},\dots, W_{m+cj, *}$ to be all zero vectors except the $j$-th entry being one. That is, we add $c$ constraints of the form $|z_j|$. And $y_{m+c(j-1)+1}=\cdots = y_{m+cj}=1$. 
\end{itemize}

\paragraph{Completeness.}
Let $\OPT$ be the optimal value of max-cut of $G$ and let $S_{\OPT}$ be the optimal subset. Then, for each $u\in S_{\OPT}$, we set $z_u=1$; and for $u\notin S_{\OPT}$, we set $z_u=-1$. For the first type constraints $|z_u+z_v|$, if $u$ and $v$ are cut by $S_{\OPT}$, then $|z_u+z_v|=0$; otherwise $|z_u+z_v|=2$. For the second type constraints $|z_j|$, all of them are satisfied by our assignment. Thus, $\|Wz-y\|_2^2 =4(m-\OPT)$. 

\paragraph{Soundness.}
Let $\eta$ be a constant such that $r<\eta < 1$,  where $r$ is the approximation lower bound in Theorem~\ref{thm:maxcut_hard}. Let $\delta=\frac{1-\eta}{10c}$.
We will show that, if there exits a $z$ such that $\|Wz-y\|_2^2 \leq \delta m'$, then we can recover a subset $S$ with cut-size $\eta m$. 

It is easy to see that the optimal solution lies in $[-1,1]^n$. Since for $z\notin [-1,1]^n$, we can always transform it to a new vector $z'\in [-1,1]^n$ such that $\|Wz'-y\|_2\leq \|Wz-y\|_2$. 

Suppose $z\in \{-1,1\}^n$ is a Boolean vector. Then, we can pick $S=\{i\in [n]: z_i=1\}$. We have the cut-size of $S$ is
\begin{align*}
    |E(S,V\backslash S)|\geq  &~ m - \delta m'/4\\
    = &~ m - \delta (1+3c/2)m/4\\
    = &~ (1-\delta /4-3c\delta/8)m\\
    \geq &~ \eta m,
\end{align*}
where the last step follows from $\delta \leq \frac{8(1-\eta)}{2+6c}$.

For general $z\in [-1,1]^n$, we first round $z$ by its sign: let $\bar{z}_i=\mathrm{sign}(z_i)$ for $i\in [n]$. We will show that
\begin{align*}
    \|W\bar{z}-y\|_2^2 - \|Wz-y\|_2^2 \leq \frac{48}{c}m
\end{align*}
which implies
\begin{align*}
    \|W\bar{z}-y\|_2^2 = &~ \|Wz-y\|_2^2 + (\|W\bar{z}-y\|_2^2 - \|Wz-y\|_2^2)\\
    \leq &~ \delta m' + \frac{48}{c}m.
\end{align*}
Then, we have the cut-size of $S$ is 
\begin{align*}
    |E(S,V\backslash S)| \geq &~ m - (\delta m' - 48m/c)/4\\
    = &~ (1-\delta/4 - 3c\delta /8 - 12/c) m \\
    \geq &~ \eta m,
\end{align*}
where the last step follows from $\delta \leq \frac{8(1-\eta-12/c)}{2+6c}$.

Let $\Delta_i := |\bar{z}_i - z_i| = 1-|z_i| \in [0,1]$. We have
\begin{align*}
    \|W\bar{z}-y\|_2^2 - \|Wz-y\|_2^2 = &~ \sum_{i=1}^{m} (\bar{z}_{u_i}+\bar{z}_{v_i})^2 - (z_{u_i}+z_{v_i})^2 + c\cdot \sum_{j=1}^n (|\bar{z}_j|-1)^2 - (|z_j|-1)^2\\
    = &~ \sum_{i=1}^{m} (\bar{z}_{u_i}+\bar{z}_{v_i})^2 - (z_{u_i}+z_{v_i})^2 - c\cdot\sum_{j=1}^n  (|z_j|-1)^2\\
    = &~ \sum_{i=1}^{m} (\bar{z}_{u_i}+\bar{z}_{v_i})^2 - (z_{u_i}+z_{v_i})^2-c\cdot \sum_{j=1}^n \Delta_j^2\\
    \leq &~ \sum_{i=1}^m 4|\Delta_{u_i} + \Delta_{u_j}|-c\cdot \sum_{j=1}^n \Delta_j^2\\
    = &~ \sum_{i=1}^n 12\Delta_i - c\Delta_i^2\\
    \leq &~ \frac{72}{c} n\\
    = &~ \frac{48}{c}m,
\end{align*}
where the first step follows by the construction of $W$ and $y$. The second step follows from $|\bar{z}_j|=1$ for all $j\in [n]$. The third step follows from the definition of $\Delta_j$. The forth step follows from $|z_{u_i}+z_{u_j}|\in [0,2]$. The fifth step follows from the degree of the graph is 3. The fifth step follows from the minimum of the quadratic function $12x-cx^2$ in $[0,1]$ is $\frac{72}{c}$. The last step follows from $m=3n/2$.

Therefore, by the completeness and soundness of reduction, if we take $\epsilon\in (0,\delta)$, Theorem~\ref{thm:maxcut_hard} implies that it is {\NP}-hard to $(1+\epsilon)$-approximate the $\ell_2$-regression, which completes the proof of the first part of the theorem.

For the furthermore part, we can use the same reduction for a 5-regular graph. By choosing proper parameters ($c$ and $\delta$), we can use Theorem~\ref{thm:maxcut_eth_hard} to rule out $2^{o(n)}$-time algorithm for $O(1)$-factor approximation. We omit the details since they are almost the same as the first part.
\end{proof}

\end{document}